\documentclass[conference]{IEEEtran}
\IEEEoverridecommandlockouts

\usepackage{cite}
\usepackage{amssymb, amsmath, amsthm, amsfonts}
\usepackage{algorithmic}
\usepackage{graphicx}
\usepackage{tikz}
\usepackage{textcomp}
\usepackage{xcolor}
\usepackage{url}
\def\BibTeX{{\rm B\kern-.05em{\sc i\kern-.025em b}\kern-.08em
    T\kern-.1667em\lower.7ex\hbox{E}\kern-.125emX}}

\newcommand{\Sec}[1]{\hyperref[sec:#1]{\S\ref*{sec:#1}}} 
\newcommand{\Eqn}[1]{\hyperref[eq:#1]{(\ref*{eq:#1})}} 
\newcommand{\Fig}[1]{\hyperref[fig:#1]{Figure~\ref*{fig:#1}}} 
\newcommand{\Tab}[1]{\hyperref[tab:#1]{Table~\ref*{tab:#1}}} 
\newcommand{\Thm}[1]{\hyperref[thm:#1]{Theorem~\ref*{thm:#1}}} 
\newcommand{\Lem}[1]{\hyperref[lem:#1]{Lemma~\ref*{lem:#1}}} 
\newcommand{\Prop}[1]{\hyperref[prop:#1]{Property~\ref*{prop:#1}}} 
\newcommand{\Cor}[1]{\hyperref[cor:#1]{Corollary~\ref*{cor:#1}}} 
\newcommand{\Def}[1]{\hyperref[def:#1]{Definition~\ref*{def:#1}}} 
\newcommand{\Alg}[1]{\hyperref[alg:#1]{Algorithm~\ref*{alg:#1}}} 
\newcommand{\Ex}[1]{\hyperref[ex:#1]{Example~\ref*{ex:#1}}} 

\newcommand{\mc}[1]{\mathcal{#1}}
\newcommand{\mb}[1]{\mathbb{#1}}

\newcommand{\veps}{\varepsilon}

\newcommand{\E}{\mathbb{E}}
\DeclareMathOperator{\KL}{KL}
\newcommand*\dif{\mathop{}\!\mathrm{d}}

\newtheorem{theorem}{Theorem}
\newtheorem{corollary}{Corollary}
\newtheorem{lemma}{Lemma}
\newtheorem{proposition}{Proposition}
\newtheorem{remark}{Remark}

\theoremstyle{definition}
\newtheorem{definition}{Definition}

\begin{document}

\title{Robust Machine Learning via Privacy/Rate-Distortion Theory}

\author{
\IEEEauthorblockN{
Ye Wang\IEEEauthorrefmark{1},
Shuchin Aeron\IEEEauthorrefmark{2},
Adnan Siraj Rakin\IEEEauthorrefmark{3},
Toshiaki Koike-Akino\IEEEauthorrefmark{1},
Pierre Moulin\IEEEauthorrefmark{4}}
\IEEEauthorblockA{\IEEEauthorrefmark{1}Mitsubishi Electric Research Laboratories, \IEEEauthorrefmark{2}Tufts University, \\
\IEEEauthorrefmark{3}Arizona State University, \IEEEauthorrefmark{4}University of Illinois at Urbana-Champaign
\\
\IEEEauthorrefmark{1}\{yewang, koike\}@merl.com,
\IEEEauthorrefmark{2}shuchin@ece.tufts.edu,
\IEEEauthorrefmark{3}asrakin@asu.edu,
\IEEEauthorrefmark{4}pmoulin@illinois.edu}
}

\maketitle

\begin{abstract}
Robust machine learning formulations have emerged to address the prevalent vulnerability of deep neural networks to adversarial examples. Our work draws the connection between optimal robust learning and the privacy-utility tradeoff problem, which is a generalization of the rate-distortion problem. The saddle point of the game between a robust classifier and an adversarial perturbation can be found via the solution of a  maximum conditional entropy problem. This information-theoretic perspective sheds light on the fundamental tradeoff between robustness and clean data performance, which ultimately arises from the geometric structure of the underlying data distribution and perturbation constraints.
\end{abstract}


\section{Introduction}

The widespread susceptibility of neural networks to adversarial examples~\cite{Szegedy-ICLR14-Intriguing, Goodfellow-ICLR15-AdvExamples} has been demonstrated through a wide variety of practical attacks~\cite{sharif2016accessorize, Kurakin-AdvExPhysicalWorld, moosavi2017universal, eykholt2018robust, athalye2018synthesizing, van2019fooling, li2019adversarial}.
This has motivated much research towards mitigating these vulnerabilities, although many earlier defenses have been shown to be ineffective~\cite{carlini2017towards, carlini2017adversarial, athalye2018obfuscated}.
We focus our attention on robust learning formulations that aim for guaranteed resiliency against the worst-case input perturbations or in a distributional sense.
Our work draws the information-theoretic connections between optimal robust learning and the privacy-utility tradeoff problem.
We utilize this perspective to shed light on the fundamental tradeoff between robustness and clean data performance, and to inspire novel algorithms for optimizing robust models.

The influential approach of~\cite{madry2017towards} proposes the robust optimization formulation given by
\begin{align*}
\min_\theta \E_{P_{X,Y}} \Big[ \max_{\delta \in \mathcal{S}} \ell(f_\theta(X + \delta), Y) \Big],
\end{align*}
where $\delta$ represents the worst-case over some set $\mathcal{S}$ of small perturbations applied to the original input $X$ of the model $f_\theta$, with the aim of maximizing the loss $\ell$ with respect to the true label $Y$.
This formulation has inspired a plethora of defenses: some that tackle the problem directly (albeit with limitations to scalability)~\cite{huang2017safety, katz2017reluplex, ehlers2017formal, cheng2017maximum, tjeng2019evaluating} and others that employ approximate bounding~\cite{wong2018provable, wong2018scaling, raghunathan2018certified, raghunathan2018semidefinite, wong2019wasserstein} or noise injection~\cite{lecuyer2018certified, li2018second, cohen2019certified} to certify robustness guarantees.

We generalize this formulation to allow stronger adversaries that may employ mixed strategies, where the perturbation can be viewed as a channel $P_{Z|X,Y}$,
while focusing our study on the fundamental optimum of the ideal robust classification game.
With the minimization over all decision rules $q(Y|Z)$ for the cross-entropy loss objective, we show the following minimax result that reduces the problem to a maximum conditional entropy problem,
\begin{align*}
&\min_{q(Y|Z)} \max_{P_{Z|X,Y} \in \mathcal{D}} \E [ - \log q(Y|Z) ] \\
&\quad = \max_{P_{Z|X,Y} \in \mathcal{D}} \min_{q(Y|Z)} \E [ - \log q(Y|Z) ] 
= \max_{P_{Z|X,Y} \in \mathcal{D}} H(Y|Z).
\end{align*}
This minimax result is established in Theorems~\ref{thm:Minimax} and~\ref{thm:equiv} in terms of the more general notion of distributional robustness, which considers the worst-case data distribution over some convex set $\mathcal{D}$.
This subsumes expected distortion constraints as a special case when $\mathcal{D}$ is a Wasserstein-ball with a suitably chosen ground metric.
For the maximum conditional entropy problem over a Wasserstein-ball constraint, we present a fixed point characterization, which exposes the interplay between the geometry of the ground cost in the Wasserstein-ball constraint, the worst-case adversarial distribution, and the given reference data distribution.

The minimax equality establishes the connection to the privacy-utility tradeoff problem~\cite{rebollo2010t, Calmon2012privacy, sankar2013utility, makhdoumi2014information, salamatian2015managing, basciftci2016privacy}, where the aim is to design a distortion-constrained data perturbation mechanism $P_{Z|X,Y}$ that maximizes the uncertainty about sensitive information $Y$ as measured by $H(Y|Z)$.
The equivalence between the maximin problem and maximum conditional entropy is used by~\cite{Calmon2012privacy} to argue that conditional entropy measures privacy against an inference attacker represented by $q$.
Figure~\ref{fig:connections} illustrates these connections.

A similar minimax result is given in~\cite{farnia2016minimax}, however with technical limitations preventing it from addressing adversarial input perturbation (see Appendix, Section~\ref{sec:TseFarniaDiffs}), and much of their development focuses on the case where the marginal distribution for $X$ remains fixed.
The similarities between the robust learning and privacy problems are noted by~\cite{hamm2017machine}, however, they only state the minimax inequality relating the two.

We examine the fundamental tradeoff between model robustness and clean data performance from our information-theoretic perspective.
This tradeoff ultimately arises from the geometric structure of the underlying data distribution and the adversarial perturbation constraints.
We illustrate these tradeoffs with the numerical analysis of a toy example.
The fundamental tradeoff between clean data and adversarial loss is also theoretically addressed by~\cite{tsipras2019robustness}.
This theory was further expanded upon by~\cite{zhang2019theoretically} and leveraged to develop an improved adversarial training defense.

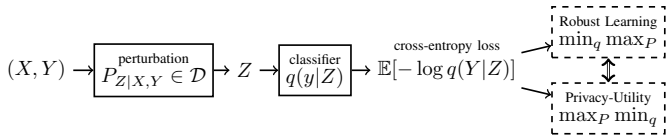
\begin{figure}
\centering
\scalebox{0.78}{
\hspace{-9pt}
\begin{tikzpicture}[rectnode/.style={rectangle, draw=black, thick}]
\node at (-0.5, 0) (xy) {$(X,Y)$};
\node[rectnode] at (1.5,0) (p) {$\overset{\text{perturbation}}{P_{Z|X, Y} \in \mathcal{D}}$};
\node at (3.05,0) (z) {$Z$};
\node[rectnode] at (4.25,0) (q) {$\overset{\text{classifier}}{q(y|Z)}$};
\node[label={[label distance=-3.5pt]\scriptsize cross-entropy loss}] at (6.5,0) (loss) {$\E[-\log q(Y|Z)]$};
\node[rectnode, dashed] at (9.25,0.65) (minmax) {$\overset{\text{Robust Learning}}{\min_q \max_P}$};
\node[rectnode, dashed] at (9.25,-0.65) (maxmin) {$\overset{\text{Privacy-Utility}}{\max_P \min_q}$};
\draw[->, thick] (xy) -- (p);
\draw[->, thick] (p) -- (z);
\draw[->, thick] (z) -- (q);
\draw[->, thick] (q) -- (loss);
\draw[<->, double] (minmax) -- (maxmin);
\draw[->, thick] (loss) -- (minmax);
\draw[->, thick] (loss) -- (maxmin);
\end{tikzpicture}
}
\caption{Robust Learning and Privacy-Utility Tradeoff problems both involve a game between a classifier and a constrained input perturbation.
The goal of robust learning is a classifier robust to the perturbation, and posed as a minimax problem.
The alternative maximin optimization captures the privacy-utility tradeoff problem, where the goal is a perturbation mechanism that hides sensitive information from an adversarial classifier aiming to recover it.
Our minimax result shows that these two problems are equivalent.
}
\label{fig:connections}
\end{figure}

\textbf{Notation:}
We use $\mathcal{P(Z|X,Y)}$ to denote the set of conditional probability distributions over $\mathcal{Z}$ given variables over the sets $\mathcal{X}$ and $\mathcal{Y}$, and $\mathcal{P(Y|X)}$ is similarly defined.

\section{Robust Machine Learning} \label{sec:RobustLearning}

The influential robust learning formulation of~\cite{madry2017towards}
addresses the worst-case attack, as given by
\begin{align} \label{eqn:pure_robust_learning}
\min_\theta \E \Bigg[ \max_{\substack{Z \in \mathcal{X}: \\ d(X,Z) \le \epsilon}} \ell(f_\theta(Z), Y) \Bigg],
\end{align}
where $d: \mathcal{X} \times \mathcal{X} \rightarrow [0, \infty]$ is some suitably chosen distortion metric (e.g., often $\ell_0$, $\ell_p$, or $\ell_\infty$ distance), and $\epsilon \geq 0$ represents the allowable perturbation.
The robust learning formulation in~\eqref{eqn:pure_robust_learning} can be viewed as a two-player, zero-sum game, where the adversary (corresponding to the inner maximization) plays second using a pure strategy by picking a fixed $Z$ subject to the distortion constraint.
We will instead consider an adversary that utilizes a mixed strategy, where $Z \in \mathcal{X} =: \mathcal{Z}$ can be a randomized function of $(X,Y)$ as specified by a conditional distribution $P_{Z|X,Y}$.
This revised formulation is given by
\begin{align} \label{eqn:mixed_robust_learning}
\min_\theta \max_{P_{Z|X,Y} \in \mathcal{D}^*_{d,\epsilon}} \E[\ell(f_\theta(Z), Y)],
\end{align}
where the expectation is over $(X, Y, Z) \sim P_{X,Y} P_{Z|X,Y}$, and the distortion limit is given by
\begin{align} \label{eqn:max_distortion}
\mathcal{D}^*_{d,\epsilon} := \{P_{Z|X,Y} \in \mathcal{P(Z|X,Y)} : \Pr[d(X,Z) \le \epsilon] = 1\}.
\end{align}
Note that under this maximum distortion constraint, allowing mixed strategies does not actually strengthen the adversary, i.e., the games in~\eqref{eqn:pure_robust_learning} and~\eqref{eqn:mixed_robust_learning} have the same value.
However, if we replace the distortion limit constraint of~\eqref{eqn:max_distortion} with an average distortion constraint, given by
\begin{align} \label{eqn:expected_distortion}
\mathcal{D}_{d,\epsilon} := \{P_{Z|X,Y} \in \mathcal{P(Z|X,Y)} : \E[d(X,Z)] \le \epsilon\},
\end{align}
then the adversary is potentially strengthened, i.e.,
\begin{align*}
\max_{P_{Z|X,Y} \in \mathcal{D}_{d,\epsilon}} \E[\ell(f_\theta(Z), Y)] \geq \max_{P_{Z|X,Y} \in \mathcal{D}^*_{d,\epsilon}} \E[\ell(f_\theta(Z), Y)].
\end{align*}

\subsection{Distributional Robustness}

Since the objective $\E[\ell(f_\theta(Z), Y)]$ only depends on the joint distribution of the variables $(Z, Y) \in \mathcal{X} \times \mathcal{Y}$, the robust learning formulation is straightforward to generalize by instead considering the maximization over an arbitrary set of joint distributions $\mathcal{D} \subset \mathcal{P(X, Y)}$.
With a change of variable (replacing $Z$ with $X$ to simplify presentation), this becomes
\begin{align} \label{eqn:general_robust_learning}
\min_\theta \max_{p \in \mathcal{D}} \E_{(X,Y) \sim p}[\ell(f_\theta(X), Y)],
\end{align}
which includes the scenarios considered in~\eqref{eqn:pure_robust_learning} through~\eqref{eqn:expected_distortion} as special cases.
However, unlike these earlier formulations,~\eqref{eqn:general_robust_learning} allows for the label $Y$ to also be potentially changed.

Another particular case for $\mathcal{D}$ is the Wasserstein-ball around a distribution $\mu \in \mathcal{P(X, Y)}$, as given by
\begin{align} \label{eqn:Wasserstein-ball}
\mathcal{D}^\mathbb{W}_\epsilon(\mu) := \{\nu \in \mathcal{P(X, Y)} : \mathbb{W}_{d}(\mu, \nu) \leq \epsilon \},
\end{align}
where $\mathbb{W}_d$ is the 1-Wasserstein distance \cite{OTAM,Villani,COT} for some ground metric (or in general a cost) $d$ on the space $\mathcal{X \times Y}$.
Recall that the 1-Wasserstein distance is given by
\begin{align*}
\mathbb{W}_d(\mu, \nu) := \inf_{\gamma \in \Gamma(\mu, \nu)} \E_\gamma \big[ d \big( (X,Y), (X',Y') \big) \big],
\end{align*}
where the set of couplings $\Gamma(\mu, \nu)$ is defined as all joint distributions with the marginals $(X,Y) \sim \mu$ and $(X',Y') \sim \nu$.
Note that maximizing over $p \in \mathcal{D}^\mathbb{W}_\epsilon(P_{X, Y})$ is equivalent to maximizing over channels $P_{X', Y' | X, Y}$ subject to the distortion expected constraint $\E \big[ d \big( (X,Y), (X',Y') \big) ] \leq \epsilon$, where $(X, Y, X', Y') \sim P_{X, Y} P_{X', Y' | X, Y}$.
Unlike the formulation considered in~\eqref{eqn:mixed_robust_learning}, this channel may also change the label $Y$.
However, if modifying $Y$ is prohibited by a cost of the form
\begin{align} \label{eqn:ground-cost}
d \big( (x,y), (x',y') \big) = \begin{cases}
d(x, x'), & \text{if } y = y',\\
\infty, & \text{otherwise},
\end{cases}
\end{align}
then the 1-Wasserstein distributionally robust formulation specializes to the earlier formulation in~\eqref{eqn:mixed_robust_learning} with the average distortion constraint given by~\eqref{eqn:expected_distortion}.
Robust-ML with Wasserstein-ball constraints is also referred to as Distributional Robust Optimization (DRO)~\cite{Blanchet16,Gao16,Gao17} and shown to be equivalent to imposing Lipschitz constraints on the classifier~\cite{Cranko20,Gao17}.
There is however no characterization, that is considered in these papers, of the optimal value of the min-max problem in this setting. 

\subsection{Optimal Robust Learning} \label{sec:optimal-learning}

The specifics of the loss function $\ell$ and model $f_\theta$ are crucial to our analysis.
Hence, we will focus specifically on learning classification models, where $X \in \mathcal{X}$ represents the data features, $Y \in \mathcal{Y} := \{1, \ldots, m\}$ represent class labels, and the model $f_\theta : \mathcal{X} \rightarrow [0,1]^m$ can be viewed as producing $q_\theta \in \mathcal{P(Y|X)}$ that aims to approximate the underlying posterior $P_{Y|X}$.
When cross-entropy is the loss function, i.e., $\ell(f_\theta(X), Y) = - \log q_\theta(Y | X)$, the expected loss, with respect to some distribution $(X, Y) \sim p = P_X P_{Y|X}$, is given by
\begin{align}
&\E_p[ - \log q_\theta(Y | X) ] \nonumber \\
&\quad = \int_{\mathcal{X}} \sum_{y \in \mathcal{Y}} P_{Y|X}(y|x) \log \frac{P_{Y|X}(y|x)}{q_\theta(y | x) P_{Y|X}(y|x)} \dif P_X(x) \nonumber \\
&\quad = \KL( P_{Y|X}(y|X) \| q_\theta(y|X) | P_X) + H(Y | X). \label{eqn:cross-entropy}
\end{align}
Thus, the principle of learning via minimizing the expected cross-entropy loss optimizes the approximate posterior $q_\theta(y | x)$ toward the underlying posterior $P_{Y|X}$, and the loss is lower bounded by the conditional entropy $H(Y | X)$, which is arguably nonzero for nontrivial classification problems.

The robust learning problem, given by
\begin{align} \label{eqn:robust_classifier_learning}
\min_\theta \max_{p \in \mathcal{D}} \E_{(X,Y) \sim p}[- \log q_\theta(Y | X)],
\end{align}
still critically depends on the specific parametric family (e.g., neural network architecture) chosen for the model $\{f_\theta\}_{\theta \in \Theta}$, which determines the corresponding parametric family of approximate posteriors, i.e., $\{q_\theta \in \mathcal{P(Y|X)}\}_{\theta \in \Theta}$.
Motivated by the ultimate meta-objective of learning the best possible robust models, we consider the idealized optimal robust learning formulation where the minimization is performed over all conditional distributions $q \in \mathcal{P(Y|Z)}$, as given by
\begin{align} \label{eqn:optimal_robust_learning}
\min_{q \in \mathcal{P(Y|Z)}} \max_{p \in \mathcal{D}} \E_{(X,Y) \sim p}[- \log q(Y | X)],
\end{align}
which clearly lower-bounds~\eqref{eqn:robust_classifier_learning}, which is specific to the particular parametric family.

\section{The Privacy-Utility Tradeoff Problem} \label{sec:PrivacyUtility}

In the information-theoretic treatment of the privacy-utility tradeoff problem~\cite{rebollo2010t, Calmon2012privacy, sankar2013utility, makhdoumi2014information, salamatian2015managing, basciftci2016privacy}, the random variables $(X, Y) \sim P_{X,Y}$ respectively denote useful and sensitive data, and the goal is to release data $Z$ produced from a randomized algorithm viewed as a channel $P_{Z|X,Y}$, while simultaneously preserving the privacy of the sensitive $Y$ and maintaining utility by conveying $X$.
Privacy is measured by $I(Y;Z)$, where smaller is better to preserve privacy.
Utility is quantified with a distortion function, $d: \mathcal{X} \times \mathcal{Z} \to [0, \infty)$, given by the particular application.
Minimizing (or limiting) the distortion $d(X, Z)$ captures the objective of maintaining the utility of the data release.
Since the useful and sensitive data $(X, Y)$ are correlated (and indeed the problem is uninteresting if they are independent),
a tradeoff naturally emerges between the two objectives of preserving privacy and utility.

\subsection{Optimal Privacy-Utility Tradeoff}

The optimal privacy-utility tradeoff problem is formulated as an information-theoretic optimization problem in~\cite{rebollo2010t, Calmon2012privacy}, and is given by
\begin{align} \label{eqn:MI-privacy-utility-tradeoff}
\mathop{\arg \min}_{P_{Z|X,Y} \in \mathcal{D}_{d, \epsilon}} I(Y;Z) = \mathop{\arg \max}_{P_{Z|X,Y} \in \mathcal{D}_{d, \epsilon}} H(Y|Z),
\end{align}
where $(X, Y, Z) \sim  P_{X,Y} P_{Z|X,Y}$, the constraint $\mathcal{D}_{d,\epsilon}$, as given in~\eqref{eqn:expected_distortion},
captures the expected distortion budget, and the equivalence follows from $I(Y; Z) = H(Y) - H(Y|Z)$ since $H(Y)$ is constant.
Similarly, one could consider the alternative maximum distortion constraint $\mathcal{D}^*_{d,\epsilon}$, given in~\eqref{eqn:max_distortion}.

\subsection{Adversarial Formulation of Privacy}

In~\cite{Calmon2012privacy}, the privacy-utility problem in~\eqref{eqn:MI-privacy-utility-tradeoff}, is derived from a broader perspective that poses privacy as maximizing the loss of an adversary that mounts a statistical inference attack attempting to recover the sensitive $Y$ from the release $Z$.
Their framework considers an adversary that can observe the release $Z$ 
and choose a conditional distribution $q \in \mathcal{P}(\mathcal{Y}|\mathcal{Z})$ to minimize its expected loss.
As observed in~\cite{Calmon2012privacy}, when cross-entropy (or ``self-information'') is the loss, we have that
\begin{align}
\min_{q \in \mathcal{P}(\mathcal{Y}|\mathcal{Z})} \E[- \log q(Y | Z)] = H(Y | Z),
\end{align}
with the optimum $q^* = p_{Y|Z}$, which follows from a derivation similar to~\eqref{eqn:cross-entropy}.
Thus, the optimal privacy-utility tradeoff given in~\eqref{eqn:MI-privacy-utility-tradeoff} is equivalent to a maximin problem,
as stated in Lemma~\ref{lem:privacy-equiv}.

\begin{lemma}[equivalence of privacy formulations~\cite{Calmon2012privacy}] \label{lem:privacy-equiv}
For any joint distribution $P_{X,Y}$ and closed, convex constraint set $\mathcal{D} \subset \mathcal{P}(\mathcal{Z} | \mathcal{X}, \mathcal{Y})$, e.g., $\mathcal{D}^*_{d,\epsilon}$ or $\mathcal{D}_{d,\epsilon}$, as given by~\eqref{eqn:max_distortion} or~\eqref{eqn:expected_distortion}, we have
\begin{align*}
& \max_{P_{Z|X,Y} \in \mathcal{D}} \min_{q \in \mathcal{P}(\mathcal{Y}|\mathcal{Z})} \E[- \log q(Y | Z)] \\
&\quad = \max_{P_{Z|X,Y} \in \mathcal{D}} H(Y|Z)
= H(Y) - \min_{P_{Z|X,Y} \in \mathcal{D}} I(Y;Z),
\end{align*}
where $(X, Y, Z) \sim  P_{X,Y} P_{Z|X,Y}$.
\end{lemma}

The privacy-utility tradeoff problem is also highly related to rate-distortion theory, which considers the efficiency of lossy data compression.
When $X = Y$, the optimization problem in~\eqref{eqn:MI-privacy-utility-tradeoff} immediately reduces to the single-letter characterization of the optimal rate-distortion tradeoff.
However, the privacy problem considers an inherently single-letter scenario, where we deal with just a single instance of the variables $(X, Y, Z)$, which could be high-dimensional, but have no restrictions placed on their statistical structure across these dimensions.

\section{Main Results -- Duality between Optimal Robust Learning and Privacy-Utility Tradeoffs}
\label{sec:Connections}

The solution to the optimal minimax robust learning problem can be found via a maximum conditional entropy problem related to the privacy-utility tradeoff problem.

\begin{theorem}
\label{thm:Minimax}
For any finite sets $\mathcal{X}$ and $\mathcal{Y}$, and closed, convex set of joint distributions $\mathcal{D} \subset \mathcal{P(X,Y)}$,
we have
\begin{align}
&\min_{q \in \mathcal{P(Y|X)}} \max_{p \in \mathcal{D}} \E[- \log q(Y | X)] \label{eqn:minimax-LHS} \\
& \qquad = \max_{p \in \mathcal{D}} \min_{q \in \mathcal{P(Y|X)}} \E[- \log q(Y | X)] \label{eqn:minimax-RHS} \\
& \qquad = \max_{p \in \mathcal{D}} H(Y|X) =: h^* \leq \log |\mathcal{Y}|, \label{eqn:minimax-entropy}
\end{align}
where the expectations and entropy are with respect to $(X, Y) \sim p$.
Further, the solutions for $q \in \mathcal{P(Y|X)}$ that minimize~\eqref{eqn:minimax-LHS} are given by
\begin{align} \label{eqn:solution}
\bigcap_{p \in \mathcal{D}} \big\{ q \in \mathcal{P(Y|X)} : \E_{(X,Y) \sim p}[- \log q(Y | X)] \leq h^* \big\} \neq \varnothing.
\end{align}
\end{theorem}
\begin{proof}
See Appendix, Section~\ref{sec:minimax_proof}.	
\end{proof}

Intuitively, the optimal minimax robust decision rule $q$ that solves~\eqref{eqn:minimax-LHS} must be consistent with the posterior $p(y|x)$ corresponding to the solution of the maximum conditional entropy problem in~\eqref{eqn:minimax-entropy}.
However, a given posterior $p(y|x)$ is well-defined only over the support of the marginal distribution of $X$, whereas the robust decision rule needs to be defined over the entire space $\mathcal{X}$.
Hence, generally, determining the robust decision rule over the entirety of $\mathcal{X}$ requires considering the solution set in~\eqref{eqn:solution}, which seems cumbersome, but can be simplified in many cases via the following corollary.

\begin{corollary} \label{cor:soln}
Under the paradigm of Theorem~\ref{thm:Minimax}, let
\begin{align*}
\mathcal{D}^* := \big \{ p \in \mathcal{D} : H(Y|X) = h^*, (X,Y) \sim p \big \}.
\end{align*}
For all $p^* \in \mathcal{D}^*$, the corresponding terms of~\eqref{eqn:solution} are given by
\begin{align*}
Q(p^*) :=&\  \big\{ q \in \mathcal{P(Y|X)} : \E_{(X,Y) \sim p^*}[- \log q(Y | X)] \leq h^* \big\} \\
=&\ \big \{ q \in \mathcal{P(Y|X)} : \forall (x,y),  q(y|x) p^*(x) = p^*(x,y) \}.
\end{align*}
Further, if
\begin{align*}
\bigcup_{p^* \in \mathcal{D}^*} \big\{ x \in \mathcal{X} : p^*(x) > 0 \big\} = \mathcal{X},
\end{align*}
then the solution set given by~\eqref{eqn:solution}, for the minimization of~\eqref{eqn:minimax-LHS}, contains exactly one point and is given by
\begin{align*}
\bigcap_{p^* \in \mathcal{D}^*} Q(p^*) = \bigcap_{p \in \mathcal{D}} Q(p).
\end{align*}
\end{corollary}

In the simplest case, if there exists a $p^* \in \mathcal{D}^*$ that has full support over $\mathcal{X}$ (in the marginal distribution for $X$), then the optimal robust decision rule that solves the minimization of~\eqref{eqn:minimax-LHS} is simply given by the posterior $p^*(y|x)$, which is defined for all $x \in \mathcal{X}$.

\subsection{Generalization to Arbitrary Alphabets}

Extending the result in the previous section to continuous $\mathcal{X}$ requires one to expand the set of allowable Markov kernels, i.e., conditional probabilities, to what is referred to as the set of generalized decision rules in statistical decision theory~\cite{strasser2011,lecam1955, cam1986asymptotic,vaart2002}.
This is because the set of Markov kernels is not compact, while the set of generalized decision rules is. 
For any $f \in C_b(\mathcal{Y})$, set of bounded continuous functions, and any bounded signed measure $\varphi$ on $\mathcal{X}$, given a mapping $q(Y|X)$ (interpret this as a measurable function $q_{x}$ over $\mc{Y}$ for each fixed $x$), define a bilinear functional via,
\begin{align}
    \beta_{q(Y|X)}(f, \varphi)  =  \int_{\mathcal{X}} \int_\mathcal{Y} f(y) q(dy|dx) \dif \varphi(x). \label{eq:bilin}
\end{align}
\begin{definition}\cite{strasser2011}
\label{def:gen_dec}
A generalized decision function is a bilinear function $\beta: C_b(\mc{Y}) \times \varphi \rightarrow \mb{R}$ that satisfies, (a) $\mbox{if}\, f\geq 0,\varphi\geq 0  \implies \beta(f,\varphi) \geq 0$, (b) $|\beta(f,\varphi)| \leq \|f\|_{\infty} \|\varphi\|_{TV}$, (c) $\beta(1,\varphi) = \|\varphi\|_{TV}$ if $\varphi \geq 0$.
\end{definition}
Define the set of generalized decision rules as the set of bi-linear functions defined via~\eqref{eq:bilin} and satisfying the properties (a), (b), (c) above.
\begin{align*}
	\mathcal{M} = \{q(Y|X): q(Y|X) \,\,\mbox{satisfies a. b. c. in Def.~\ref{def:gen_dec} via~\eqref{eq:bilin}} \}
 \end{align*}
Applying these results, we obtain the following theorem for the case of general alphabets $\mc{X}$. Note that in contrast to Theorem \ref{thm:Minimax}, here the results hold with $\inf, \sup$ instead of $\min, \max$. 
\begin{theorem} \label{thm:equiv}
Under the paradigm of Theorem \ref{thm:Minimax}, for continuous alphabets $\mathcal{X}$ and discrete $\mathcal{Y}$, 
\begin{align}
&\inf_{q \in \mathcal{M}} \sup_{p \in \mc{D}} \E_p[- \log q(Y | X)] = \sup_{p \in \mc{D} } H(Y|X)
\end{align}
\end{theorem}
\begin{proof}
Using the fact that the set $\mathcal{M}$ is convex and compact for the weak topology (Theorem 42.3, \cite{strasser2011}), that the function $\E_p[- \log q(Y | X)]$ is convex in $q$ for all $q \in \mc{M}$, and applying the minimax theorem \cite{pollard-minimax}, we have
 \begin{align}
	&\inf_{q \in \mathcal{M}} \sup_{p \in \mc{D}} \E_p[- \log q(Y | X)] =  \sup_{p \in \mc{D}}\inf_{q \in \mc{M}} \E_p[- \log q(Y | X)],
\end{align}
and noting that $\inf_{q \in \mc{M}} \mb{E}_{p}[-\log q(Y|X)]  = H(Y|X)$, the result follows.
Hence, even in the case of continuous alphabets, the worst case \emph{algorithm-independent} adversarial perturbation can be found by solving $\sup_{p \in \mc{D} } H(Y|X)$.
\end{proof}

\section{Implications of the Main Results}

\subsection{Necessity of Stochastic Perturbation}

In the original robust learning formulation, as given in~\eqref{eqn:pure_robust_learning}, the attacker is restricted to a pure strategy, and this is not suboptimal (i.e., this game has the same value as the mixed strategy formulation given by~\eqref{eqn:mixed_robust_learning}), since the attacker has the advantage of ``playing second'' with the inner maximization.
However, we emphasize that the original formulation given by~\eqref{eqn:pure_robust_learning}, even in the basic case of optimal robust classification, is not necessarily a saddle point problem, that is,
\begin{align}
&\min_{q \in \mathcal{P}(\mathcal{Y}|\mathcal{Z})} \E \Bigg[ \max_{\substack{Z \in \mathcal{X}: \\ d(X,Z) \le \epsilon}} - \log q(Y|Z) \Bigg] \label{eqn:deterministic-minimax} \\
&\quad \geq \max_{\substack{g : \mathcal{X} \times \mathcal{Y} \rightarrow \mathcal{X} \\ d(X, g(X, Y)) \le \epsilon}} \min_{q \in \mathcal{P}(\mathcal{Y}|\mathcal{Z})} \E \Big[ - \log q\big(Y | g(X,Y)\big) \Big] \label{eqn:deterministic-maximin}
\end{align}
will often be a strict inequality due to the determinism of the attack mapping $g$.
In contrast, our minimax result of Theorem~\ref{thm:Minimax} establishes that with a stochastic attacker (or, more generally, distributional robustness constrained to a convex set), such as formulated in~\eqref{eqn:mixed_robust_learning}, swapping the min and max does not disadvantage the attacker for ``playing first''.

We illustrate the necessity of a stochastic attacker with the following example.
Consider $\mathcal{X} = \mathcal{Y} = \{0, 1, 2, 3, 4\}$, where $P_{X,Y}(x,y) = 1/3$ for $(x, y) \in \{(0, 0), (2, 2), (4, 4)\}$, and let $\epsilon = 1$ be the distortion limit under the metric $d(x,z) = |x - z|$.
For this setup, the optimal stochastic attack will clearly lie within the family parameterized by $\alpha \in [0, 1]$ and given by
\begin{align*}
p^\alpha_{Z | X} (z | x) :=
\begin{cases}
1, & \text{if } (x, z) \in \{ (0, 1), (4, 3) \}, \\
\alpha, & \text{if } (x, z) = (2, 1), \\
1 - \alpha, & \text{if } (x, z) = (2, 3),
\end{cases}
\end{align*}
however, the optimal deterministic attack is limited to only $\alpha$ equal to zero or one.
The optimal stochastic attack that solves~\eqref{eqn:minimax-entropy}, and hence also~\eqref{eqn:minimax-LHS} and~\eqref{eqn:minimax-RHS} due to Theorem~\ref{thm:Minimax} and Corollary~\ref{cor:soln}, is found at $\alpha = 0.5$ yielding the optimal value of $h^* = h_2(1/3)$, where $h_2(p) := - p \log(p) - (1-p) \log (1-p)$ is the binary entropy function.
For deterministic attacks, the optimal value of~\eqref{eqn:deterministic-minimax} is also $h_2(1/3)$, however, the optimal value of~\eqref{eqn:deterministic-maximin} is equal to $(2/3) \log(2) < h_2(1/3)$.

\subsection{Tradeoffs between Robustness vs Clean Data Loss} \label{sec:tradeoffs}

A natural question to ask is whether robustness comes at a price.
It has been observed empirically that robust models will underperform on clean data in comparison to conventional, non-robust models.
To understand why this is fundamentally unavoidable, we examine the loss for robust and non-robust models in combination with clean data or adversarial attack.

Let $\mu \in \mathcal{D}$ denote the unperturbed (clean data) distribution within the set of potential adversarial attacks $\mathcal{D}$.
For a given decision rule $q \in \mathcal{P(Y | X)}$ and distribution $\nu = \nu_X \nu_{Y|X} \in \mathcal{P(X, Y)}$, recall that the cross-entropy loss is given by~\eqref{eqn:cross-entropy} as
\begin{align*}
\mathcal{L}(\nu, q) := &\ \E_p[ - \log q(Y | X) ] \\
= &\ H_\nu(Y|X) + \KL( \nu_{Y|X} \| q(y|X) | \nu_X).
\end{align*}
The baseline loss of the ideal non-robust model for clean data is given by
$\min_q \mathcal{L}(\mu, q) = H_\mu(Y|X)$.
Under adversarial attack, the ideal loss of the robust model is given by Theorem~\ref{thm:Minimax} as
\begin{align*}
\min_q \max_{\nu \in \mathcal{D}} \mathcal{L}(\nu, q) = \max_{\nu \in \mathcal{D}} H_\nu(Y|X).
\end{align*}
The loss of a robust model $q^*$ that solves~\eqref{eqn:minimax-LHS}, as characterized by~\eqref{eqn:solution}, under the clean data distribution $\mu$ is given by
\begin{align*}
\mathcal{L}(\mu, q^*) = H_\mu(Y|X) + \KL( \mu_{Y|X} \| q^*(y|X) | \mu_X).
\end{align*}
The KL-divergence term must be finite, since we have
\begin{align*}
H_{\mu}(Y|X) &= \min_q \mathcal{L}(\mu, q) \leq \mathcal{L}(\mu, q^*) \\
&\leq \min_q \max_{\nu \in \mathcal{D}} \mathcal{L}(\nu, q) = \max_{\nu \in \mathcal{D}} H_\nu(Y|X),
\end{align*}
where the second inequality follows from $q^*$ being the minimax solution.

\begin{figure}[ht]
    \centering
    \includegraphics[width=0.499\textwidth]{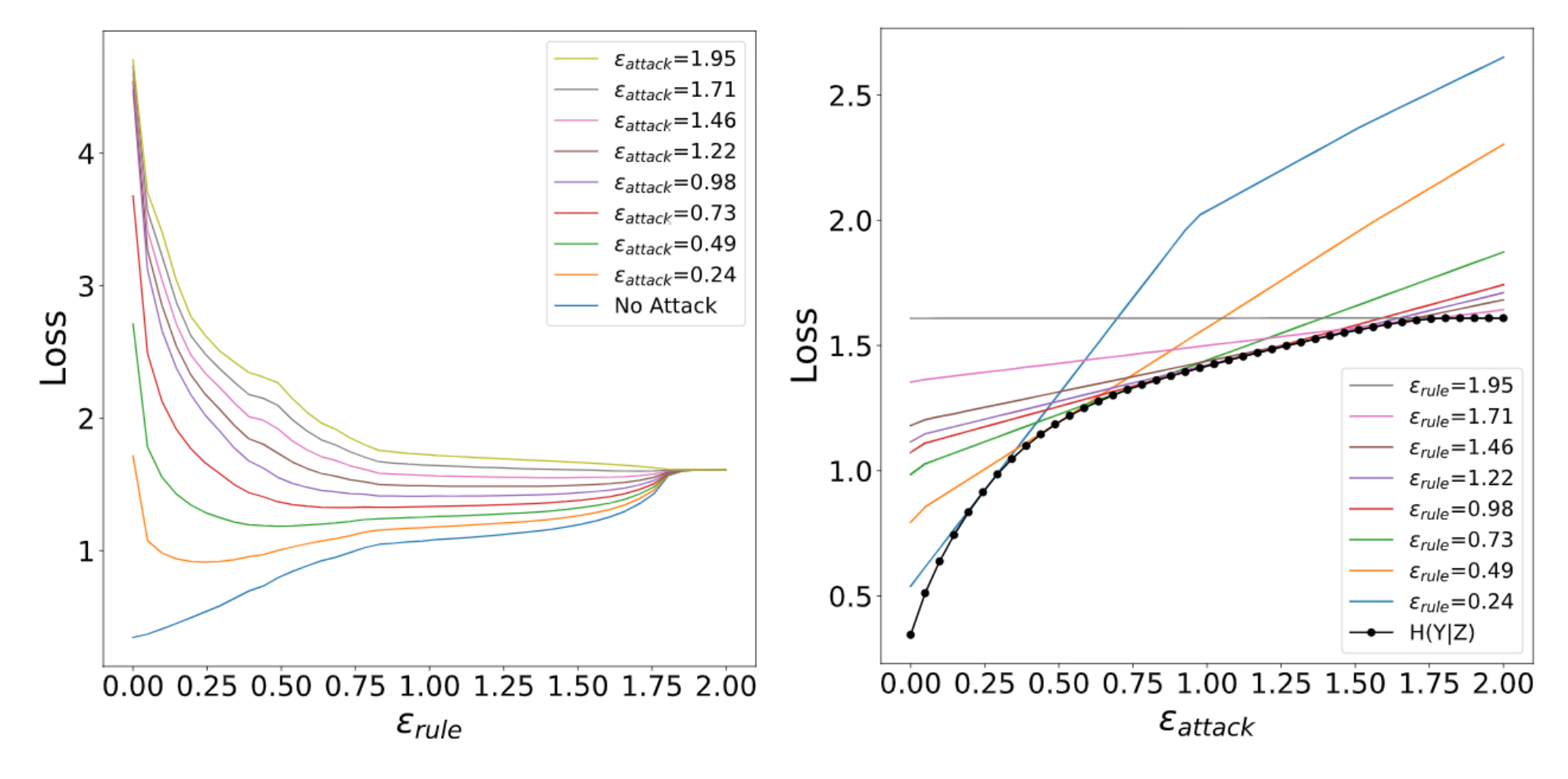}
    \caption{\textit{Left:} Loss as a function of decision rule, varying $\epsilon_\text{rule}$, and across attacks varying $\epsilon_\text{attack}$. \textit{Right:} Loss as a function of attack distortion, varying $\epsilon_\text{attack}$, and across decision rules varying $\epsilon_\text{rule}$.}
    \label{fig:loss}
\end{figure}

We numerically evaluate these tradeoffs by considering a family of Wasserstein-ball constraint sets $\mathcal{D}(\epsilon)$, as given by~\eqref{eqn:Wasserstein-ball}, with varying radius $\epsilon \geq 0$ around a distribution $\mu$ over finite alphabets $\mathcal{X} = \mathcal{Y} = \{1, \ldots, 5\}$.
The ground metric is of the form given in~\eqref{eqn:ground-cost}, which effectively limits the perturbation to only changing $X$ within an expected squared-distance distortion constraint of $\epsilon$, as equivalent to~\eqref{eqn:expected_distortion}.
The distribution $\mu$ was randomly chosen, and has entropies $H_\mu(Y) \approx 1.6$ and $H_\mu(Y|X) \approx 0.34$ (in nats).

Leveraging Theorem~\ref{thm:Minimax} and Corollary~\ref{cor:soln}, we numerically solve for the robust decision rules,
\begin{align*}
q^*_{\epsilon_\text{rule}} = \mathop{\arg \min}_{q \in \mathcal{P(Y | X)}} \max_{\nu \in \mathcal{D}(\epsilon_\text{rule})} \mathcal{L}(\nu, q),
\end{align*}
across a range distortion constraints $\epsilon_\text{rule} \in [0, 2]$.
In combination with each decision rule, we consider the loss under attacks at varying distortion limits $\epsilon_\text{attack} \in [0, 2]$, as given by
\begin{align*}
\mathcal{L}(\epsilon_\text{attack}, \epsilon_\text{rule}) := \max_{\nu \in \mathcal{D}(\epsilon_\text{attack})} \mathcal{L}(\nu, q^*_{\epsilon_\text{rule}}).
\end{align*}

Figure~\ref{fig:loss} plots the loss $\mathcal{L}(\epsilon_\text{attack}, \epsilon_\text{rule})$ across the combination of $\epsilon_\text{attack}$ and $\epsilon_\text{rule}$.
On the $\emph{left}$ of Figure~\ref{fig:loss}, each curve is a fixed attack distortion $\epsilon_\text{attack}$, over which the decision rule $q^*_{\epsilon_\text{rule}}$ is varied, with the optimal loss obtained when $\epsilon_\text{rule} = \epsilon_\text{attack}$.
As $\epsilon_\text{rule}$ increases, the loss for all curves converge to $H_\mu(Y)$.
In the $\emph{right}$ of Figure~\ref{fig:loss}, the dotted black curve is the maximum conditional entropy $H_\nu (Y|X)$ over $\nu \in \mathcal{D}(\epsilon_\text{attack})$ at each $\epsilon_\text{attack}$, which corresponds to the ideal robust loss when $\epsilon_\text{rule} = \epsilon_\text{attack}$.
The other curves are each a fixed decision rule $q^*_{\epsilon_\text{rule}}$, over which the attack distortion $\epsilon_\text{attack}$ is varied, which exhibits suboptimal loss for mismatched $\epsilon_\text{rule} \neq \epsilon_\text{attack}$.
The beginning of each curve, at $\epsilon_\text{attack} = 0$, is the clean data loss for each rule, and we can see that clean data loss is degraded as robustness for higher distortions $\epsilon_\text{attack}$ is improved.
In the extreme of a decision rule designed to be robust for very high $\epsilon_\text{rule} = 1.95$, the loss is uniformly equal to $H_\mu(Y)$ across all $\epsilon_\text{attack}$, since this robust decision rule $q^*_{1.95}$ only simply guesses the prior $\mu_Y$.

\subsection{Fixed point characterization of the worst case perturbation}

We consider the particular case when $\mathcal{D}$ is the Wasserstein-ball around a distribution $\mu \in \mathcal{P(X, Y)}$:
\begin{align*}
\mathcal{D}^\mathbb{W}_\epsilon(\mu) := \{\nu \in \mathcal{P(X, Y)} : \mathbb{W}_{d}(\mu, \nu) \leq \epsilon \},
\end{align*}
and derive the necessary conditions for optimality for the solution to $\sup_{\nu \in \mc{D} } H_\nu(Y|X)$, where by the subscript in the conditional entropy we highlight the fact that the conditional entropy is computed under the joint distribution $\nu$. To this end we adopt a Lagrangian viewpoint and we assume that $\mc{X}$ and $\mc{Y}$ are continuous bounded and compact sets, but the result can be seen to hold true when $\mc{X}$ is continuous and $\mc{Y}$ is discrete. The result is summarized in the Theorem below.
\begin{theorem}
\label{thm:fixed_point}
If the cost $d$ is continuous with continuous first derivative and the distribution $\mu(x,y)$ is supported on the whole of the domain $\mc{X \times Y}$, the optimal solution to $\arg \min_{\nu} \mb{W}_d(\nu, \mu) - \lambda H_\nu(Y|X)$ for some $\lambda >0$ satisfies, 
\begin{align}
\varphi_{\nu \rightarrow \mu}(x,y) = \lambda (\log (\nu(x,y)) - u(y) \log \nu(x)) + C,
\end{align}
where $\varphi_{\nu \rightarrow \mu}(x,y)$ is the Kantorovich Potential \footnote{Kantorovich Potential is the variable of optimization in the dual problem to the optimal transport problem. We refer the reader to \cite{OTAM,Villani} and \cite{COT} for these definitions and notions related to theory of Optimal Transport.}corresponding to the optimal solution to the transport problem from $\nu$ to $\mu$ under the ground cost $d$, capital $C$ is a constant, $u(y)$ is a uniform distribution over $\mc{Y}$, and $\nu(x) = \int_y \nu(x,y)$ is the marginal distribution under the joint $\nu(x,y)$. 
\end{theorem}
\begin{proof}
See Appendix, Section~\ref{sec:fixed_point_proof}.	
\end{proof}

This characterization ties closely the geometry of the perturbations (as reflected via the Kantorovich Potential) with the worst case distribution that maximizes the conditional entropy. 
The algorithmic implications of this fixed point relation will be undertaken in future work.

\bibliographystyle{IEEEtran}
\bibliography{refs}

\begin{thebibliography}{10}
\providecommand{\url}[1]{#1}
\csname url@samestyle\endcsname
\providecommand{\newblock}{\relax}
\providecommand{\bibinfo}[2]{#2}
\providecommand{\BIBentrySTDinterwordspacing}{\spaceskip=0pt\relax}
\providecommand{\BIBentryALTinterwordstretchfactor}{4}
\providecommand{\BIBentryALTinterwordspacing}{\spaceskip=\fontdimen2\font plus
\BIBentryALTinterwordstretchfactor\fontdimen3\font minus
  \fontdimen4\font\relax}
\providecommand{\BIBforeignlanguage}[2]{{%
\expandafter\ifx\csname l@#1\endcsname\relax
\typeout{** WARNING: IEEEtran.bst: No hyphenation pattern has been}%
\typeout{** loaded for the language `#1'. Using the pattern for}%
\typeout{** the default language instead.}%
\else
\language=\csname l@#1\endcsname
\fi
#2}}
\providecommand{\BIBdecl}{\relax}
\BIBdecl

\bibitem{Szegedy-ICLR14-Intriguing}
\BIBentryALTinterwordspacing
C.~Szegedy, W.~Zaremba, I.~Sutskever, J.~Bruna, D.~Erhan, I.~Goodfellow, and
  R.~Fergus, ``Intriguing properties of neural networks,'' in
  \emph{International Conference on Learning Representations}, 2014. [Online].
  Available: \url{http://arxiv.org/abs/1312.6199}
\BIBentrySTDinterwordspacing

\bibitem{Goodfellow-ICLR15-AdvExamples}
\BIBentryALTinterwordspacing
I.~Goodfellow, J.~Shlens, and C.~Szegedy, ``Explaining and harnessing
  adversarial examples,'' in \emph{International Conference on Learning
  Representations}, 2015. [Online]. Available:
  \url{http://arxiv.org/abs/1412.6572}
\BIBentrySTDinterwordspacing

\bibitem{sharif2016accessorize}
M.~Sharif, S.~Bhagavatula, L.~Bauer, and M.~K. Reiter, ``Accessorize to a
  crime: Real and stealthy attacks on state-of-the-art face recognition,'' in
  \emph{Proceedings of the 2016 ACM SIGSAC Conference on Computer and
  Communications Security}, 2016, pp. 1528--1540.

\bibitem{Kurakin-AdvExPhysicalWorld}
\BIBentryALTinterwordspacing
A.~Kurakin, I.~Goodfellow, and S.~Bengio, ``Adversarial examples in the
  physical world,'' \emph{ICLR Workshop}, 2017. [Online]. Available:
  \url{https://arxiv.org/abs/1607.02533}
\BIBentrySTDinterwordspacing

\bibitem{moosavi2017universal}
S.-M. Moosavi-Dezfooli, A.~Fawzi, O.~Fawzi, and P.~Frossard, ``Universal
  adversarial perturbations,'' in \emph{Proceedings of the IEEE conference on
  computer vision and pattern recognition}, 2017, pp. 1765--1773.

\bibitem{eykholt2018robust}
K.~Eykholt, I.~Evtimov, E.~Fernandes, B.~Li, A.~Rahmati, C.~Xiao, A.~Prakash,
  T.~Kohno, and D.~Song, ``Robust physical-world attacks on deep learning
  visual classification,'' in \emph{Proceedings of the IEEE Conference on
  Computer Vision and Pattern Recognition}, 2018, pp. 1625--1634.

\bibitem{athalye2018synthesizing}
A.~Athalye, L.~Engstrom, A.~Ilyas, and K.~Kwok, ``Synthesizing robust
  adversarial examples,'' in \emph{International Conference on Machine
  Learning}, 2018, pp. 284--293.

\bibitem{van2019fooling}
W.~Van~Ranst, S.~Thys, and T.~Goedem{\'e}, ``Fooling automated surveillance
  cameras: adversarial patches to attack person detection,'' in \emph{CVPR
  Workshop on The Bright and Dark Sides of Computer Vision: Challenges and
  Opportunities for Privacy and Security}, 2019.

\bibitem{li2019adversarial}
J.~B. Li, F.~R. Schmidt, and J.~Z. Kolter, ``Adversarial camera stickers: A
  physical camera attack on deep learning classifier,'' \emph{arXiv preprint
  arXiv:1904.00759}, 2019.

\bibitem{carlini2017towards}
N.~Carlini and D.~Wagner, ``Towards evaluating the robustness of neural
  networks,'' in \emph{2017 IEEE Symposium on Security and Privacy (SP)}, 2017,
  pp. 39--57.

\bibitem{carlini2017adversarial}
------, ``Adversarial examples are not easily detected: Bypassing ten detection
  methods,'' in \emph{Proceedings of the 10th ACM Workshop on Artificial
  Intelligence and Security}, 2017, pp. 3--14.

\bibitem{athalye2018obfuscated}
A.~Athalye, N.~Carlini, and D.~Wagner, ``Obfuscated gradients give a false
  sense of security: Circumventing defenses to adversarial examples,'' in
  \emph{International Conference on Machine Learning}, 2018, pp. 274--283.

\bibitem{madry2017towards}
\BIBentryALTinterwordspacing
A.~Madry, A.~Makelov, L.~Schmidt, D.~Tsipras, and A.~Vladu, ``Towards deep
  learning models resistant to adversarial attacks,'' in \emph{International
  Conference on Learning Representations ({ICLR})}, 2018. [Online]. Available:
  \url{https://arxiv.org/abs/1706.06083}
\BIBentrySTDinterwordspacing

\bibitem{huang2017safety}
X.~Huang, M.~Kwiatkowska, S.~Wang, and M.~Wu, ``Safety verification of deep
  neural networks,'' in \emph{International Conference on Computer Aided
  Verification}, 2017, pp. 3--29.

\bibitem{katz2017reluplex}
G.~Katz, C.~Barrett, D.~L. Dill, K.~Julian, and M.~J. Kochenderfer, ``Reluplex:
  An efficient smt solver for verifying deep neural networks,'' in
  \emph{International Conference on Computer Aided Verification}, 2017, pp.
  97--117.

\bibitem{ehlers2017formal}
R.~Ehlers, ``Formal verification of piece-wise linear feed-forward neural
  networks,'' in \emph{International Symposium on Automated Technology for
  Verification and Analysis}, 2017, pp. 269--286.

\bibitem{cheng2017maximum}
C.-H. Cheng, G.~N{\"u}hrenberg, and H.~Ruess, ``Maximum resilience of
  artificial neural networks,'' in \emph{International Symposium on Automated
  Technology for Verification and Analysis}, 2017, pp. 251--268.

\bibitem{tjeng2019evaluating}
V.~Tjeng, K.~Xiao, and R.~Tedrake, ``Evaluating robustness of neural networks
  with mixed integer programming,'' in \emph{International Conference on
  Learning Representations}, 2019.

\bibitem{wong2018provable}
E.~Wong and Z.~Kolter, ``Provable defenses against adversarial examples via the
  convex outer adversarial polytope,'' in \emph{International Conference on
  Machine Learning}, 2018, pp. 5283--5292.

\bibitem{wong2018scaling}
E.~Wong, F.~Schmidt, J.~H. Metzen, and J.~Z. Kolter, ``Scaling provable
  adversarial defenses,'' in \emph{Advances in Neural Information Processing
  Systems}, 2018, pp. 8400--8409.

\bibitem{raghunathan2018certified}
A.~Raghunathan, J.~Steinhardt, and P.~Liang, ``Certified defenses against
  adversarial examples,'' in \emph{International Conference on Machine
  Learning}, 2018.

\bibitem{raghunathan2018semidefinite}
A.~Raghunathan, J.~Steinhardt, and P.~S. Liang, ``Semidefinite relaxations for
  certifying robustness to adversarial examples,'' in \emph{Advances in Neural
  Information Processing Systems}, 2018, pp. 10\,877--10\,887.

\bibitem{wong2019wasserstein}
E.~Wong, F.~R. Schmidt, and J.~Z. Kolter, ``Wasserstein adversarial examples
  via projected sinkhorn iterations,'' \emph{arXiv preprint arXiv:1902.07906},
  2019.

\bibitem{lecuyer2018certified}
M.~Lecuyer, V.~Atlidakis, R.~Geambasu, D.~Hsu, and S.~Jana, ``Certified
  robustness to adversarial examples with differential privacy,'' \emph{arXiv
  preprint arXiv:1802.03471}, 2018.

\bibitem{li2018second}
B.~Li, C.~Chen, W.~Wang, and L.~Carin, ``Second-order adversarial attack and
  certifiable robustness,'' \emph{arXiv preprint arXiv:1809.03113}, 2018.

\bibitem{cohen2019certified}
J.~M. Cohen, E.~Rosenfeld, and J.~Z. Kolter, ``Certified adversarial robustness
  via randomized smoothing,'' \emph{arXiv preprint arXiv:1902.02918}, 2019.

\bibitem{rebollo2010t}
D.~Rebollo-Monedero, J.~Forne, and J.~Domingo-Ferrer, ``From t-closeness-like
  privacy to postrandomization via information theory,'' \emph{IEEE
  Transactions on Knowledge and Data Engineering}, vol.~22, no.~11, pp.
  1623--1636, 2010.

\bibitem{Calmon2012privacy}
F.~d.~P. Calmon and N.~Fawaz, ``Privacy against statistical inference,'' in
  \emph{2012 50th Annual Allerton Conference on Communication, Control, and
  Computing (Allerton)}, 2012, pp. 1401--1408.

\bibitem{sankar2013utility}
L.~Sankar, S.~R. Rajagopalan, and H.~V. Poor, ``Utility-privacy tradeoffs in
  databases: An information-theoretic approach,'' \emph{IEEE Transactions on
  Information Forensics and Security}, vol.~8, no.~6, pp. 838--852, 2013.

\bibitem{makhdoumi2014information}
A.~Makhdoumi, S.~Salamatian, N.~Fawaz, and M.~M{\'e}dard, ``From the
  information bottleneck to the privacy funnel,'' in \emph{2014 IEEE
  Information Theory Workshop (ITW 2014)}, 2014, pp. 501--505.

\bibitem{salamatian2015managing}
S.~Salamatian, A.~Zhang, F.~du~Pin~Calmon, S.~Bhamidipati, N.~Fawaz, B.~Kveton,
  P.~Oliveira, and N.~Taft, ``Managing your private and public data: Bringing
  down inference attacks against your privacy,'' \emph{IEEE Journal of Selected
  Topics in Signal Processing}, vol.~9, no.~7, pp. 1240--1255, 2015.

\bibitem{basciftci2016privacy}
Y.~O. Basciftci, Y.~Wang, and P.~Ishwar, ``On privacy-utility tradeoffs for
  constrained data release mechanisms,'' in \emph{2016 Information Theory and
  Applications Workshop (ITA)}, 2016, pp. 1--6.

\bibitem{farnia2016minimax}
F.~Farnia and D.~Tse, ``A minimax approach to supervised learning,'' in
  \emph{Advances in Neural Information Processing Systems}, 2016, pp.
  4240--4248.

\bibitem{hamm2017machine}
J.~Hamm and A.~Mehra, ``Machine vs machine: Minimax-optimal defense against
  adversarial examples,'' \emph{arXiv preprint arXiv:1711.04368}, 2017.

\bibitem{tsipras2019robustness}
D.~Tsipras, S.~Santurkar, L.~Engstrom, A.~Turner, and A.~Madry, ``Robustness
  may be at odds with accuracy,'' in \emph{International Conference on Learning
  Representations}, 2019.

\bibitem{zhang2019theoretically}
H.~Zhang, Y.~Yu, J.~Jiao, E.~Xing, L.~El~Ghaoui, and M.~Jordan, ``Theoretically
  principled trade-off between robustness and accuracy,'' in
  \emph{International Conference on Machine Learning}, 2019, pp. 7472--7482.

\bibitem{OTAM}
\BIBentryALTinterwordspacing
F.~Santambrogio, \emph{Optimal Transport for Applied Mathematicians: Calculus
  of Variations, PDEs and Modeling}.\hskip 1em plus 0.5em minus 0.4em\relax
  Springer, 2015. [Online]. Available:
  \url{https://www.math.u-psud.fr/~filippo/OTAM-cvgmt.pdf}
\BIBentrySTDinterwordspacing

\bibitem{Villani}
C.~Villani, \emph{Optimal Transport: Old and New}.\hskip 1em plus 0.5em minus
  0.4em\relax Springer, Berlin, Heidelberg, 2009.

\bibitem{COT}
\BIBentryALTinterwordspacing
G.~Peyré and M.~Cuturi, ``Computational optimal transport,'' \emph{Foundations
  and Trends in Machine Learning}, vol. 11 (5-6), pp. 355--602, 2019. [Online].
  Available: \url{https://arxiv.org/abs/1803.00567}
\BIBentrySTDinterwordspacing

\bibitem{Blanchet16}
J.~Blanchet and K.~Murthy, ``{Quantifying Distributional Model Risk Via Optimal
  Transport},'' \emph{SSRN Electronic Journal}, 2016.

\bibitem{Gao16}
R.~Gao and A.~J. Kleywegt, ``{Distributionally Robust Stochastic Optimization
  with Wasserstein Distance},'' \emph{arXiv preprint arXiv:1604.02199}, 2016.

\bibitem{Gao17}
\BIBentryALTinterwordspacing
R.~Gao, X.~Chen, and A.~J. Kleywegt, ``Wasserstein distributional robustness
  and regularization in statistical learning,'' \emph{CoRR}, vol.
  abs/1712.06050, 2017. [Online]. Available:
  \url{http://arxiv.org/abs/1712.06050}
\BIBentrySTDinterwordspacing

\bibitem{Cranko20}
Z.~Cranko, Z.~Shi, X.~Zhang, R.~Nock, and S.~Kornblith, ``{Generalised
  Lipschitz Regularisation Equals Distributional Robustness},'' \emph{arXiv
  preprint arXiv:2002.04197}, 2020.

\bibitem{strasser2011}
H.~Strasser, \emph{Mathematical theory of statistics: statistical experiments
  and asymptotic decision theory}.\hskip 1em plus 0.5em minus 0.4em\relax
  Walter de Gruyter, 2011, vol.~7.

\bibitem{lecam1955}
\BIBentryALTinterwordspacing
L.~LeCam, ``An extension of {Wald}'s theory of statistical decision
  functions,'' \emph{Ann. Math. Statist.}, vol.~26, no.~1, pp. 69--81, 03 1955.
  [Online]. Available: \url{https://doi.org/10.1214/aoms/1177728594}
\BIBentrySTDinterwordspacing

\bibitem{cam1986asymptotic}
\BIBentryALTinterwordspacing
L.~Cam, \emph{Asymptotic Methods in Statistical Decision Theory}, ser. Springer
  series in statistics.\hskip 1em plus 0.5em minus 0.4em\relax Springer My Copy
  UK, 1986. [Online]. Available:
  \url{https://books.google.com/books?id=BcDxoAEACAAJ}
\BIBentrySTDinterwordspacing

\bibitem{vaart2002}
\BIBentryALTinterwordspacing
A.~v.~d. Vaart, ``The statistical work of lucien le cam,'' \emph{Ann.
  Statist.}, vol.~30, no.~3, pp. 631--682, 06 2002. [Online]. Available:
  \url{https://doi.org/10.1214/aos/1028674836}
\BIBentrySTDinterwordspacing

\bibitem{pollard-minimax}
\BIBentryALTinterwordspacing
D.~Pollard, \emph{Asymptopia}.\hskip 1em plus 0.5em minus 0.4em\relax
  Unpublished manuscript, 2003. [Online]. Available:
  \url{http://www.stat.yale.edu/~pollard/Courses/602.spring07/MmaxThm.pdf}
\BIBentrySTDinterwordspacing

\bibitem{rudin1964analysis}
W.~Rudin, \emph{Principles of Mathematical Analysis}.\hskip 1em plus 0.5em
  minus 0.4em\relax McGraw-Hill, 1964.

\end{thebibliography}

\section{Proof of Theorem~\ref{thm:Minimax}} \label{sec:minimax_proof}

\begin{proof}
The relations in~\eqref{eqn:minimax-entropy} and the existence of the maximums and minimum in~\eqref{eqn:minimax-RHS} and~\eqref{eqn:minimax-entropy} follow from a straightforward generalization of Lemma~\ref{lem:privacy-equiv}.
The rest of the proof follows the same general steps as the proof of a generalized minimax theorem given by~\cite{pollard-minimax}, except adapted for minimums and maximums rather than infimums and supremums.

For convenience, we define
\begin{align}
f(p, q) &:= \E_{(X,Y) \sim p}[- \log q(Y | X)] - h^*, \label{eqn:conv-fdef} \\
Q(p) &:= \big\{ q \in \mathcal{P(Y|X)} : f(p, q) \leq 0 \big\}. \label{eqn:Q-sets}
\end{align}
Note that $f(p, q)$ is linear in $p$ for fixed $q$, and convex in $q$ for fixed $p$.
Further, for all $p \in \mathcal{D}$, $\min_{q \in \mathcal{P(Y|X)}} f(p, q) \leq 0$ and $Q(p)$ is compact, convex, and nonempty.

We only need to show that~\eqref{eqn:minimax-LHS} is less than or equal to~\eqref{eqn:minimax-RHS}, which would follow from $\cap_{p \in \mathcal{D}} Q(p) \neq \varnothing$, which is equivalent to~\eqref{eqn:solution}.
Since, each $Q(p)$ is compact, it is sufficient to show that $\cap_{p \in \mathcal{D}_0} Q(p) \neq \varnothing$ for every finite subset $\mathcal{D}_0 \subset \mathcal{D}$~\cite[Thm.~2.36]{rudin1964analysis}.
We will first show this for any two-point set $\mathcal{D}_0 = \{p_1, p_2\}$, and later extend this to every finite set through an inductive argument.

Suppose $Q(p_1) \cap Q(p_2) = \varnothing$, then a contradiction would occur if we can show that there exists $\alpha \in [0,1]$ such that for all $q \in \mathcal{P(Y|X)}$,
\begin{align} \label{eqn:alpha-contra}
(1-\alpha) f(p_1, q) + \alpha f(p_2, q) > 0,
\end{align}
since then $\min_{q \in \mathcal{P(Y|X)}} f(p_\alpha, q) > 0$, where $p_\alpha := (1-\alpha) p_1 + \alpha p_2$.

For $q \notin Q(p_1) \cup Q(p_2)$, we immediately have~\eqref{eqn:alpha-contra}, since both $f(p_1, q) > 0$ and $f(p_2, q) > 0$.
For~\eqref{eqn:alpha-contra} to hold for all $q \in Q(p_1)$, we must require
\begin{align} \label{eqn:alpha-sup}
\alpha > \sup_{q_1 \in Q(p_1)} \frac{-f(p_1, q_1)}{f(p_2, q_1) - f(p_1, q_1)}.
\end{align}
The supremum is $\geq 0$, since $f(p_1, q_1) \leq 0$ and $f(p_2, q_1) > 0$, from the assumption $Q(p_1) \cap Q(p_2) = \varnothing$.
For~\eqref{eqn:alpha-contra} to hold for all $q \in Q(p_2)$, we must also require
\begin{align} \label{eqn:alpha-inf}
\alpha < \inf_{q_2 \in Q(p_2)} \frac{f(p_1, q_2)}{f(p_1, q_2) - f(p_2, q_2)}.
\end{align}
The infimum is $\leq 1$, since $f(p_2, q_2) \leq 0$ and $f(p_1, q_2) > 0$, from the assumption $Q(p_1) \cap Q(p_2) = \varnothing$.
Thus, an $\alpha$ satisfying both~\eqref{eqn:alpha-sup} and~\eqref{eqn:alpha-inf} exists if and only if for all $q_1 \in Q(p_1)$ and $q_2 \in Q(p_2)$,
\begin{align*}
\frac{-f(p_1, q_1)}{f(p_2, q_1) - f(p_1, q_1)} < \frac{f(p_1, q_2)}{f(p_1, q_2) - f(p_2, q_2)},
\end{align*}
or equivalently,
\begin{align} \label{eqn:alpha-exists}
f(p_1, q_1) f(p_2, q_2) < f(p_1, q_2) f(p_2, q_1).
\end{align}

Since~\eqref{eqn:alpha-exists} is immediate if either $f(p_1, q_1) = 0$ or $f(p_2, q_2) = 0$, we need only consider when both $f(p_1, q_1) < 0$ and $f(p_2, q_2) < 0$.
Define $\theta \in (0,1)$ such that
\begin{align} \label{eqn:theta_1}
(1 - \theta) f(p_1, q_1) + \theta f(p_1, q_2) = 0,
\end{align}
and let $q_\theta := (1 - \theta) q_1 + \theta q_2$.
Since $f$ is convex in $q$, $f(p_1, q_\theta) \leq 0$, which implies that $q_\theta \in Q(p_1)$ hence $q_\theta \notin Q(p_2)$ (since we assumed that they are disjoint), which further implies that
\begin{align} \label{eqn:theta_2}
(1-\theta) f(p_2, q_1) + \theta f(p_2, q_2) \geq f(p_2, q_\theta) > 0.
\end{align}
Thus, by combining~\eqref{eqn:theta_1} and~\eqref{eqn:theta_2},
\begin{align*}
\frac{-f(p_1, q_1)}{f(p_1, q_2)} = \frac{\theta}{1 - \theta} < \frac{f(p_2, q_1)}{-f(p_2, q_2)},
\end{align*}
which implies~\eqref{eqn:alpha-exists} and the existence of $\alpha$, which contradicts the assumption that $Q(p_1) \cap Q(p_2) = \varnothing$.

The pairwise result $Q(p_1) \cap Q(p_2) \neq \varnothing$ implies that for any finite set $\mathcal{D}_0 = \{p_1, \ldots, p_m\}$, $Q(p_1) \cap Q(p_i) \neq \varnothing$ for $i = 2, \ldots, m$.
Then, we can repeat the argument starting from~\eqref{eqn:conv-fdef} with $q$ further restricted to $Q(p_1)$, i.e., replacing $\mathcal{P(Y|X)}$ in subsequent steps with $Q(p_1)$, which effectively redefines~\eqref{eqn:Q-sets} with $Q'(p) := Q(p) \cap Q(p_1)$, and eventually leads to $Q(p_1) \cap Q(p_2) \cap Q(p_i) \neq \varnothing$ for $i = 3, \ldots, m$.
Thus, repeating this argument further yields that $\cap_{p \in \mathcal{D}_0} Q(p) \neq \varnothing$ for any finite subset $\mathcal{D}_0 \subset \mathcal{D}$, which, as argued earlier, implies~\eqref{eqn:solution}.
\end{proof}

\section{Proof of Theorem~\ref{thm:fixed_point}} \label{sec:fixed_point_proof}
All the proof steps assume continuous and compact $\mc{X,Y}$ but it is easy to see that the steps hold true for discrete and finite $\mc{Y}$ and continuous $\mc{X}$. We begin with the following definition that is taken from Chapter 7 in \cite{OTAM}. 
\begin{definition}
Given a functional $F(\rho): \mc{P} \rightarrow \mb{R}$, if $\rho$ is a regular point\footnote{See Chapter 7, \cite{OTAM} for definition of a regular point.} of $F$, and for any perturbation $\chi = \rho - \tilde{\rho}, \tilde{\rho} \in \mc{P} \cap L_c^{\infty}({\Omega})$, one calls $\frac{\delta F}{\delta \rho}(\rho)$ the first variation of $F(\rho)$ if 
\begin{align*}
	\frac{d}{d \varepsilon} F(\rho + \varepsilon \chi)|_{\varepsilon =0} = \int \frac{\delta F}{\delta \rho}(\rho)  d\chi
\end{align*}  
\end{definition}
It can be seen that the first variations are unique up a constant. The proof then follows from the following two lemmas. 

\begin{lemma}\cite{OTAM} The first variation of a the optimal transport cost $\mb{W}_d(\nu, \mu)$ with respect to $\nu$ is given by the Kontorovich potential, $\varphi_{\nu \rightarrow \nu}$, provided it is unique. A sufficient condition for uniqueness of $\varphi_{\nu \rightarrow \nu}$ is that the cost $c$ is continuous with continuous first derivative and $\mu$ is supported on the whole of the domain.
\end{lemma}

\begin{lemma} \label{lem:firstvar_Ent}
The first variation of the conditional entropy function defined by 
\begin{align*}
 	H_\nu(Y|X) = \int \nu(x,y) \log \frac{\nu(x,y)}{\int_y \nu(x,y)} dx dy , 
 \end{align*}
 is given by $\log (\nu(x,y)) - u(y) \log \nu(x)$, where $u(y)$ is a uniform distribution over $\mc{Y}$ and $\nu(x)$ is the marginal over $\mc{X}$ under the joint $\nu(x,y)$.
\end{lemma}
\begin{proof} \textbf{Notation}: In the following to be concise and avoid a cumbersome notation we will often not explicitly write $\chi(x,y)$ but just use $\chi$. On the other hand we will keep explicit the notation $\nu(x,y)$ so as to not lose sight of it. 

	By definition consider a perturbation $\veps \chi $ around $\nu$ and let us look at 
\begin{align*}
	\frac{d}{d \veps} & \int (\nu(x,y) + \veps \chi)  \log \frac{(\nu(x,y) + \veps \chi)}{\int_y (\nu(x,y)+ \veps \chi)} dx dy \\
	& = \frac{d}{d \veps} \int (\nu(x,y) + \veps \chi)  \log (\nu(x,y) + \veps \chi) dx dy \\
	&\quad - \frac{d}{d \veps} \int (\nu(x,y) + \veps \chi)  \log (\int_y (\nu(x,y)+ \veps \chi)) dx dy \\
	& = \frac{d}{d \veps} \int (\nu(x,y) + \veps \chi)  \log (\nu(x,y) + \veps \chi) dx dy \\
	&\quad - \frac{d}{d \veps} \int (\nu(x) + \veps f(\chi))  \log (\nu(x) + \veps f(\chi))) dx dy
\end{align*}
where $f(\chi) = \int_y \chi(x,y) dy$.
Let us focus on the first term. 
\begin{align*}
	\frac{d}{d \veps} & \int (\nu(x,y) + \veps \chi)  \log (\nu(x,y) + \veps \chi)  dx dy  \nonumber \\
	& = \int \frac{d}{d \veps} \nu(x,y)  \log (\nu(x,y) + \veps \chi) dx dy \\
	&\quad + \int \frac{d}{d \veps} \veps \chi \log (\nu(x,y) + \veps \chi) dx dy \\
	& = \int \nu(x,y) \frac{\chi}{(\nu(x,y) + \veps \chi)} dx dy \\
	&\quad + \int \log(\nu(x,y) + \veps \chi) \chi + \int \frac{\veps \chi^2}{(\nu(x,y) + \veps \chi)} dx dy
\end{align*}
From this we conclude that, 
\begin{align*}
	\frac{d}{d \veps} & \int (\nu(x,y) + \veps \chi)  \log (\nu(x,y) + \veps \chi) dx dy |_{\veps =0} \\
	&= \int (1 + \log (\nu(x,y)) \chi dx dy
\end{align*}
Now let us look at the second term. Following the same arguments as for the first term we have, 
\begin{align*}
	\frac{d}{d \veps} & \int (\nu(x) + \veps f(\chi))  \log (\nu(x) + \veps f(\chi))) dx dy |_{\veps =0} \\
	&= \int (\log(\nu(x)) + 1) f(\chi) dx dy
\end{align*}
Now we note that,
\begin{align*}
	\int (\log \nu(x) + 1) f(\chi) dx dy = \int (u(y) \log\nu(x)  + 1)\chi dx dy
\end{align*}
where $u(y)$ is the uniform distribution over $\mathcal{Y}$. Therefore we have, 
\begin{align*}
	\frac{d}{d \veps} H_{\nu + \veps\chi}(Y|X)|_{\veps = 0} = \int (\log (\nu) - u(y)\log\nu(x)) \chi dx dy
\end{align*}

\end{proof}

\section{Differences versus Farnia and Tse's Minimax Result}
\label{sec:TseFarniaDiffs}

The strong version of the minimax result from~\cite[Thm.~1.B]{farnia2016minimax} requires a continuity assumption on $f(p, q)$, as defined in~\eqref{eqn:conv-fdef}, with respect to $p \in \mathcal{D}$.
This continuity assumption is stated in the following Proposition~\ref{prop:continuity} and is generally false, except for particular choices of $\mathcal{D}$ that may limit the applicability of their minimax result toward addressing general adversarial examples.
Our minimax results in Theorem~\ref{thm:Minimax} and Theorem~\ref{thm:equiv} avoid this assumption and its limitations.

\begin{proposition} \label{prop:continuity}
If a sequence $(p_n)_{n=1}^\infty \in \mathcal{D}$ converges to $p_0 \in \mathcal{D}$, and $q_n := \arg \min_q f(p_n, q)$, then for any $p \in \mathcal{D}$, $f(p, q_n)$ converges to $f(p, q_0)$.
\end{proposition}

\begin{remark} If the marginal distribution for $X$ is fixed over all joint distributions in $\mathcal{D}$, then Proposition~\ref{prop:continuity} is true.
Much of the developments in~\cite{farnia2016minimax} are constructed within this assumption.
\end{remark}

\begin{remark}
For general $\mathcal{D}$ where the marginal distribution for $X$ may vary, Proposition~\ref{prop:continuity} may be false, as shown with the following example.
\end{remark}

Let $\mathcal{X} = \mathcal{Y} = \{0, 1\}$, and $\mathcal{D}$ be all joint distributions over $\mathcal{X} \times \mathcal{Y}$.
Consider the sequence of distributions
\[
p_n(x, y) := \begin{cases}
1/2, & \text{if } (x,y) = (0,0), \\
(n-1)/2n, & \text{if } (x,y) = (0,1), \\
0, & \text{if } (x,y) = (1,0), \\
1/2n, & \text{if } (x,y) = (1,1),
\end{cases}
\]
for which the associated optimal decision rules are equivalent to the posteriors, as given by
\[
q_n := \mathop{\arg \min}_{q \in \mathcal{P(Y|X)}} f(p_n, q) \equiv p_n(y = 1 | x) = \begin{cases}
\frac{n-1}{2n - 1}, & \text{if } x = 0, \\
1, & \text{if } x = 1.
\end{cases}
\]
Also, consider the similar sequence
\[
p'_n(x, y) := \begin{cases}
(n-1)/2n, & \text{if } (x,y) = (0,0), \\
1/2, & \text{if } (x,y) = (0,1), \\
1/2n, & \text{if } (x,y) = (1,0), \\
0, & \text{if } (x,y) = (1,1),
\end{cases}
\]
and its associated optimal decision rules and posteriors
\[
q'_n := \mathop{\arg \min}_{q \in \mathcal{P(Y|X)}} f(p'_n, q) \equiv p'_n(y = 1 | x) = \begin{cases}
\frac{n}{2n - 1}, & \text{if } x = 0, \\
0, & \text{if } x = 1.
\end{cases}
\]
Note that both sequences converge to the same distribution,
\[
p_0(x,y) = \begin{cases}
1/2, & \text{if } x = 0, \\
0, & \text{if } x = 1.
\end{cases}
\]
However, the corresponding optimal decision rule $q_0 := \arg \min_q f(p_0, q)$ is not unique, and constrained only by $q_0(y | x = 0) = 1/2$, while $q_0(y | x = 1)$ may be arbitrary.
For Proposition~\ref{prop:continuity} to be true, it would be required, for any $p \in \mathcal{D}$, that both $f(p, q_n)$ and $f(p, q'_n)$ converge to $f(p, q_0)$, however, there does not exist a $q_0$ such that both simultaneously converge to $f(p, q_0)$.
For $f(p, q_n)$ to converge to $f(p, q_0)$, it would be required that $q_0(y = 1 | x = 1) = 1$, while for $f(p, q'_n)$ to converge to $f(p, q_0)$, it would be required that $q_0(y = 1 | x = 1) = 0$.

\end{document}